\documentclass[11pt]{article}

\usepackage[T1]{fontenc}
\usepackage[margin=1in]{geometry}
\usepackage{times}
\usepackage{amsmath,amssymb,amsthm,booktabs,tabularx,graphicx}
\newtheorem{theorem}{Theorem}[section]
\newtheorem{proposition}[theorem]{Proposition}

\usepackage{float}
\usepackage{algorithm}
\usepackage{algpseudocode}
\usepackage{needspace}
\usepackage{microtype}
\usepackage[round,authoryear]{natbib}
\usepackage[hidelinks]{hyperref}
\hypersetup{hidelinks}
\usepackage{url}
\hypersetup{
  pdftitle={PrT-Flow: Probability-Tractable Flow Matching for Reinforcement Learning},
  pdfauthor={Congzhong Sun and Ning Tan}
}

\title{PrT-Flow: Probability-Tractable Flow Matching for Reinforcement Learning}
\author{Congzhong Sun\thanks{Email: \href{mailto:sunczh3@mail2.sysu.edu.cn}{\texttt{sunczh3@mail2.sysu.edu.cn}}.}
  \quad Ning Tan\thanks{Corresponding author. Email: \href{mailto:tann5@mail.sysu.edu.cn}{\texttt{tann5@mail.sysu.edu.cn}}.}\\
  School of Computer Science and Engineering, Sun Yat-sen University}
\date{}

\begin{document}
\maketitle

\begin{abstract}
Diffusion models and Flow Matching can represent flexible distributions
over robot actions, but many reinforcement learning algorithms also
require the probability of each action. We introduce PrT-Flow
(Probability-Tractable Flow Matching), which computes probability density
while retaining multimodal, multi-step action generation. Its
Jacobian-Tractable Transformer (JTT) makes each sampling step invertible
and enables efficient computation of its Jacobian determinant.
We train these finite steps with Conditional Euler-step Training
(CET), which supervises their displacement over one step during sampling
rather than their instantaneous velocity. We prove that, under certain
assumptions, CET has the same minimizers as fitting the Euler-solver
updates of Flow Matching. PrT-Flow achieves 96\% success on Robomimic Lift after
300 epochs of imitation learning, matching the diffusion policy
reference, and 96\% success on OGBench Cube-Triple after 600k online
steps. These results demonstrate that invertible networks with
tractable probability densities can learn effective robot policies and support
SAC-style actor--critic training.
Our code is available at
\url{https://github.com/Sunlight-zero/PrT-Flow}.
\end{abstract}

\begin{figure}[!t]
 \centering
 \includegraphics[width=\textwidth]{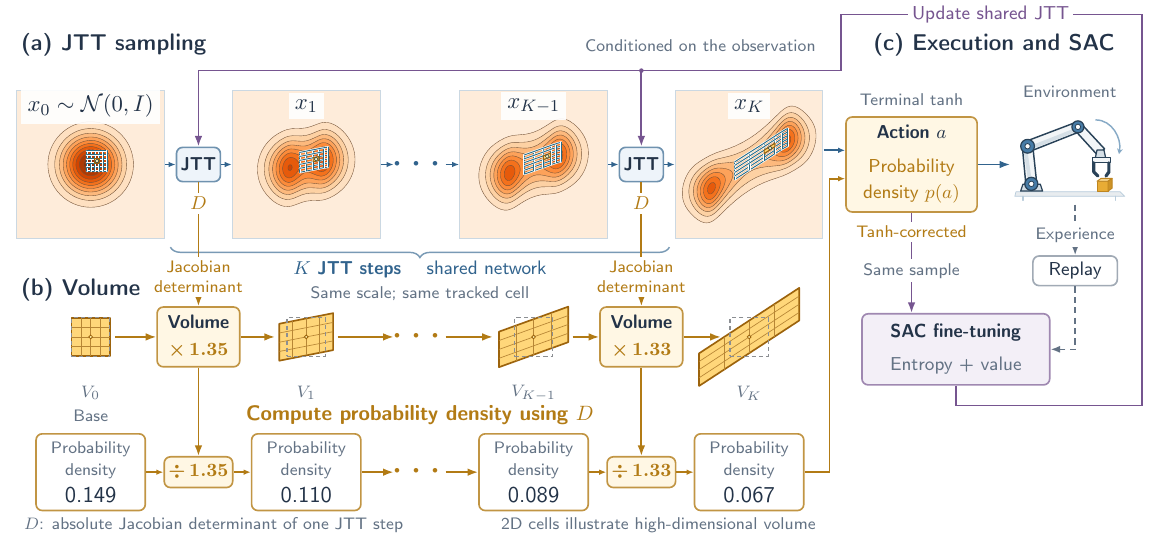}
 \caption{The overview of PrT-Flow. 
(a) Conditioned on an observation, the Jacobian-Tractable Transformer
(JTT) transforms standard Gaussian noise through $K$ invertible
steps of size $h=1/K$.
Each step returns both the updated sample and the Jacobian determinant $D$, without constructing a dense Jacobian matrix.
(b) The deformed grid illustrates
how the probability density of a sample evolves under these steps.
At each step, JTT computes the absolute Jacobian determinant $D$. The local volume is scaled by $D$ and the corresponding density is divided by $D$, allowing us to track the probability density of a sample.
(c) A final tanh transformation maps the generated sample to a bounded action chunk $a$ with its probability density $p(a)$. SAC-style actor–critic training then uses sampled actions, their log-densities, and learned Q-values to update the FM policy.}
 \label{fig:method}
\end{figure}

\section{Introduction}

Learning robot manipulation from demonstrations requires a policy
that maps observations, such as camera images or robot states, to
actions. For a given observation, several motions may
accomplish the same task, yet their average may not correspond to
any valid motion. Therefore, the policy should represent a distribution
over actions. With action chunking, it predicts a sequence of actions
from the current observation and replans from a new observation
after executing a specified number of actions.

Given an observation $c$, a stochastic policy treats an action chunk
$A$ as a random vector following the conditional distribution $\pi(\cdot\mid c)$.
Diffusion Policy learns this distribution from robot demonstrations
using a conditional diffusion model~\citep{chi2023diffusion}.
Flow Matching offers a related approach to probability transport,
learning to move noise toward demonstration actions along a prescribed
probability path~\citep{lipman2023flowmatching,liu2023rectified}.
PointFlowMatch uses this formulation in place of diffusion to learn
robot manipulation policies~\citep{chisari2024pointflowmatch}.
Both approaches use flexible generative models to capture the
multiple valid behaviors that a robot may exhibit.

Imitation learning (IL) and reinforcement learning (RL) provide
complementary ways to learn robot policies: IL trains a policy to
reproduce demonstrations, while RL uses task rewards to improve
behavior through interaction. In the actor--critic framework widely
used for continuous control, the actor represents the policy and
the critic estimates state or action values. Many RL algorithms
also require policy probabilities. Soft Actor-Critic (SAC) uses
the log probability of a sampled action in its entropy
objective~\citep{haarnoja2018sac}, while Proximal Policy Optimization
(PPO) uses the probability ratio for an action under the updated and previous
policies~\citep{schulman2017ppo}. Therefore, a generative policy must
provide both action samples and their probabilities to support
these algorithms.

Flow Matching provides a sampling procedure, but computing the
probability density of its actions requires additional work.
Continuous normalizing flows track density changes by integrating
the divergence of the velocity field~\citep{chen2018neuralode}.
However, a robot policy generates actions through finitely many
numerical steps, whose endpoint distribution may differ from
that of the continuous flow. We seek to compute the density of
each action chunk under this finite sampling procedure as the
chunk is generated.

We introduce \emph{PrT-Flow} (Probability-Tractable Flow Matching;
Figure~\ref{fig:method}), which parameterizes each sampling step
with an invertible network whose Jacobian determinant is analytically
tractable. The
\emph{Jacobian-Tractable Transformer} (JTT) combines causal
conditioning with monotone scalar transformations. Each forward
pass returns an updated sample and the corresponding determinant,
allowing the policy to compute the probability of denoised actions
using the change-of-variables formula. We train these steps with
\emph{Conditional Euler-step Training} (CET), which supervises their
displacement over one sampling step. Each training example supervises
a single step along an interpolation path between noise and a demonstration endpoint; at deployment, the policy composes
the steps to generate an action chunk. PrT-Flow thus combines the
computable probabilities of normalizing flows with local
regression along Flow Matching paths, supporting multimodal,
multi-step action generation while providing the action log
probabilities used by SAC.

The causal and invertibility constraints in JTT may affect its
ability to represent effective behavior. We therefore evaluate
PrT-Flow in both imitation learning and offline-to-online reinforcement
learning on several benchmarks.
PrT-Flow achieves high task success in these settings.

\needspace{4\baselineskip}
We summarize our contributions as follows.
\begin{enumerate}
\item We propose PrT-Flow and its JTT architecture, which compute the probability
density of an action during multi-step generation using tractable
per-step Jacobian determinants.
\item We introduce CET, a training objective whose minimizers coincide with
those of marginal Flow Matching over the same family of finite
maps. We also show its equivalence to regression against Euler-solver
updates of the marginal velocity field.
\item Our experiments on Lift and Cube-Triple show that the
invertible architecture supports high success rates in both
imitation learning and SAC-style online control.
\end{enumerate}
Our code is available at
\url{https://github.com/Sunlight-zero/PrT-Flow}.

\section{Related Work}
\label{sec:related-work}

\subsection{Diffusion and Flow Matching Policies}
As discussed in the Introduction, Diffusion Policy generates action sequences by iterative
denoising~\citep{chi2023diffusion}. Flow Matching and rectified flow
learn velocity fields that transport noise to data~\citep{lipman2023flowmatching,liu2023rectified},
and PointFlowMatch applies conditional Flow Matching to robot
manipulation~\citep{chisari2024pointflowmatch}. PrT-Flow retains
multi-step generation to preserve the multimodal ability of Diffusion
Policy and Flow Matching policies while providing the action probabilities
needed for RL.

\subsection{Reinforcement Learning for Generative Policies}
DPPO fine-tunes diffusion policies with PPO over an augmented
denoising decision process~\citep{ren2025dppo}. Its PPO ratios use
probabilities of individual denoising transitions, rather than the marginal
probability of the final action chunk. FQL distills a Flow Matching
behavior policy into a separate one-step actor optimized with
Q-values~\citep{park2025fql}. FPO constructs proxy probability ratios
from conditional Flow Matching losses~\citep{mcallister2025fpo}, while
RLDT improves a pretrained flow policy through transport
updates~\citep{lei2026rldt}. These approaches enable policy improvement
without directly evaluating the exact probability of the final action
samples produced by the multi-step generator. PrT-Flow instead
provides the final action probability directly for policy optimization.

\subsection{Density Evaluation for Policy Optimization}
Normalizing-flow policies have been used in SAC~\citep{ward2019flowactor,mazoure2020sacnf}
and dexterous fine-tuning with action-chunked critics~\citep{yang2026sernf}.
NFPO uses stabilized Real NVP couplings to provide action probabilities
for PPO ratios~\citep{shi2026nfpo}; PrT-Flow instead learns
multi-step maps through local Flow Matching regression
with CET.

SAC-Flow uses noise-augmented sampling and a joint probability density
over intermediate states as an entropy surrogate for offline-to-online
SAC~\citep{yu2025sacflow}. PrT-Flow obtains the final action density
directly by applying deterministic, invertible Euler-solver steps after
sampling the initial Gaussian noise. We mainly compare the performance
of our methods to SAC-Flow in Section~\ref{sec:experiments}.

GenPO++ uses a reversible solver with paired consecutive states,
whose fixed Jacobian determinants cancel in PPO-style joint probability
ratios~\citep{hu2026genpopp}. PrT-Flow computes the final action
probability density without an auxiliary solver-history variable.

FALCON learns interval transport through average-velocity regression,
uses cycle consistency to encourage invertibility, and evaluates
discrete-map Jacobian determinants for Boltzmann
sampling~\citep{rehman2025falcon}. CET instead directly supervises
finite maps against Euler targets at the actual sampling step size,
while JTT guarantees invertibility by construction and computes
determinants from analytic scalar derivatives.

Related architectures include coupling and spline flows~\citep{dinh2017density,durkan2019neural},
causal Transformers~\citep{zhai2025tarflow}, and learned orthogonal
projections~\citep{dai2021sinf}. JTT combines learned orthogonal
coordinates with causal monotone splines and step-size conditioning.
RegFlow regresses against precomputed endpoint
transports~\citep{rehman2026regflow}, whereas CET supervises local
finite steps. Continuous normalizing flows and maximum-entropy
Flow Matching RL track probability density through velocity-field
divergence~\citep{chen2018neuralode,zhang2025maxentfm}, whereas
PrT-Flow accumulates the log-determinants of its actual sampling maps.

\section{Preliminaries}
\label{sec:preliminaries}
\subsection{Robot Policies and Action Chunks}
\label{sec:policy-coordinates}
At environment time $\tau$, the policy receives a window of
observations $(o_{\tau-T_o+1},\ldots,o_\tau)$, where $T_o$ is the
observation horizon. We denote this observation condition by
$c_\tau$, or simply $c$ when the context is clear.
Each action is a vector $\boldsymbol a_\tau\in\mathbb R^{D_a}$.
An action-chunk policy predicts $T_p$ consecutive actions, executes
the first $T_a$, and then replans from the latest observations~\citep{chi2023diffusion}. We flatten
the predicted chunk into a single vector:
\[
 a=\operatorname{vec}(\boldsymbol a_\tau,\ldots,
                      \boldsymbol a_{\tau+T_p-1})\in\mathbb R^d,
 \qquad d=T_pD_a.
\]
We denote the conditional probability density of the predicted action
chunk by $\pi_\theta(a\mid c)$, where $\theta$ denotes the policy parameters.

The demonstration dataset $\mathcal D$ pairs each observation condition
$c$ with a recorded action chunk $a_D$ in the environment's bounded
control coordinates $[-1,1]$. The generative model operates in
$\mathbb R^d$ and maps its endpoint $u$ to an action through
$a=\tanh u$. The corresponding demonstration endpoint is
$u_D=\operatorname{atanh}\tilde a_D$, where $\tilde a_D$ is the
demonstration action after the boundary preprocessing described in
Appendix~\ref{app:protocol}. We use $t\in[0,1]$ for time within the
generative process, distinct from environment time $\tau$.

\subsection{Conditional Flow Matching}
Given a training pair $(c,u_D)$, draw Gaussian noise
$z\sim\mathcal N(0,I_d)$ and a time $t\sim\mathrm{Unif}[0,1]$
independently. Let $\phi$ denote the standard Gaussian probability density.
The straight path from the noise sample to the demonstration endpoint
has position and velocity
\begin{equation}
 X_t=(1-t)z+tu_D,\qquad V=u_D-z.
 \label{eq:cfm-path}
\end{equation}
Conditional Flow Matching fits these path velocities with a network
$v_\theta:\mathbb R^d\times[0,1]\times\mathcal C\to\mathbb R^d$,
where $\mathcal C$ is the observation condition space, by minimizing
\begin{equation}
 \mathcal L_{\mathrm{CFM}}=
 \mathbb E_{(c,u_D),z,t}\frac{1}{d}
 \|v_\theta(X_t,t,c)-V\|_2^2,
 \label{eq:cfm-standard}
\end{equation}
where the expectation over $(c,u_D)$ follows the demonstration distribution.
With an MSE loss, the network learns to predict the average target velocity
for a given state, time, and observation.
We use the straight-path formulation of
\citet{lipman2023flowmatching,liu2023rectified}; for a broader
treatment of the foundations, conditional objectives, and
implementation, see \emph{Flow Matching Guide and Code} by
\citet{lipman2024guide}.

\subsection{Sampling by Numerical Integration}
\label{sec:numerical-sampling}
To generate an action for a fixed observation, Flow Matching draws
a Gaussian sample $z$ and solves an ODE from the learned velocity field:
\begin{equation}
 \frac{d x_t}{dt}=v_\theta(x_t,t,c),\qquad x_0=z.
 \label{eq:ode}
\end{equation}
With $K$ equal steps of size $h=1/K$ and grid points $t_k=kh$,
we can approximately solve this ODE via an Euler solver:
\begin{equation}
 x_{k+1}=x_k+h\,v_\theta(x_k,t_k,c),\qquad k=0,\ldots,K-1.
 \label{eq:euler}
\end{equation}
Here $x_k$ is the numerical state at $t_k$. After $K$ steps, the
endpoint $u=x_K$ is mapped to the action chunk. Throughout the paper,
$X_t$ denotes a point constructed directly on a training path,
whereas $x_k$ denotes a state reached by the sampler. The distribution
of generated actions is determined by this finite sequence of
numerical updates.

\subsection{Soft Actor-Critic}
Soft Actor-Critic (SAC) is an off-policy algorithm that learns a policy
and an action-value function $Q_\psi(c,a)$ from transitions stored in a
replay buffer $\mathcal B$~\citep{haarnoja2018sac}. Its policy
objective combines value maximization with entropy regularization:
\begin{equation}
 \mathcal L_{\mathrm{SAC}}=
 \mathbb E_{c\sim\mathcal B,\,a\sim\pi_\theta(\cdot\mid c)}
 [\alpha\log\pi_\theta(a\mid c)-Q_\psi(c,a)],
 \label{eq:sac-standard}
\end{equation}
where $\psi$ denotes the critic parameters and $\alpha>0$ sets the
entropy weight. The policy must provide both an action sample and its
log probability, which also enters the critic's soft Bellman target.
Section~\ref{sec:sac} gives the action-chunk objectives used in our
experiments.

\section{PrT-Flow}
\label{sec:method}
PrT-Flow generates action chunks by repeatedly transforming noisy actions and tracking the change in probability density at each step.
The Jacobian-Tractable Transformer (JTT) implements these Euler-solver updates,
returning the updated state together with the update's Jacobian determinant. Both
computations are differentiable with respect to the network parameters,
allowing SAC to optimize the policy through sampled actions and their
log probabilities. Conditional Euler-step Training (CET) learns the
Euler-solver updates by supervising their displacements along demonstration paths.
We first explain how to compute probability densities during sampling,
then describe the network architecture and training procedure.

\subsection{Motivation}
\label{sec:finite-map}
Consider the Euler update from Section~\ref{sec:numerical-sampling},
$x_{k+1}=x_k+h\,v_\theta(x_k,t_k,c)$. Viewed as a transformation of
the state $x_k$, it defines $F_h(x,t,c)=x+h\,v_\theta(x,t,c)$, which we
call an \emph{Euler-step}. The sampler starts from Gaussian noise and
applies $K$ such transformations. If, for fixed $h,t,c$, each
$F_h(\cdot,t,c)$ is a piecewise continuously differentiable bijection
with a nonzero Jacobian determinant almost everywhere, the probability density
after one step follows from the change-of-variables formula:
\begin{equation}
 p_{k+1}(x_{k+1}\mid c)
 =\frac{p_k(x_k\mid c)}
 {|\det J_{F_h}(x_k,t_k,c)|},
 \qquad x_{k+1}=F_h(x_k,t_k,c).
 \label{eq:change-variables}
\end{equation}
Here $p_k$ is the probability density before the transformation, and $J_{F_h}$ is the
Jacobian of $F_h(\cdot,t,c)$. Its absolute determinant measures the
local change in volume: expansion lowers the density, while contraction
raises it. Taking logarithms makes these corrections additive across
sampling steps.

For an arbitrary velocity network $v_\theta$, differentiating the Euler-step
with respect to $x$, while holding $h,t,c$ fixed, gives
\begin{equation}
 J_{F_h}(x,t,c)=I+hJ_{v_\theta}(x,t,c).
 \label{eq:euler-jacobian}
\end{equation}
Each probability density correction therefore requires
$\det(I+hJ_{v_\theta})$, which generally entails additional Jacobian
and determinant calculations. Moreover, a finite Euler step need not
be globally invertible. We address both issues by parameterizing
$F_{\theta,h}$ directly as an invertible network with an analytic
determinant. Trained on Euler-step targets along Flow Matching paths,
this transformation serves as a \emph{learned Euler-step}.
Figure~\ref{fig:method} illustrates how it generates samples and tracks
their densities.

For step size $h=1/K$, our sampler is
\begin{equation}
 x_0=z\sim\mathcal N(0,I_d),\quad
 x_{k+1}=F_{\theta,h}(x_k,t_k,c),\quad
 t_k=kh,\quad u=x_K,\quad a=\tanh u.
 \label{eq:chain}
\end{equation}
The final transformation, $a=\tanh u$, maps the generated endpoint
into the bounded action space. Writing
$\ell_k=\log|\det J_{F_{\theta,h}}(x_k,t_k,c)|$, we obtain the sampled
action chunk's log probability density:
\begin{equation}
 \log\pi_\theta(a\mid c)=
 \log\phi(z)-\sum_{k=0}^{K-1}\ell_k
 -\sum_{j=1}^d\log(1-a_j^2).
 \label{eq:density}
\end{equation}
The three terms account for the Gaussian probability density, the
determinants of the learned updates, and the final tanh transformation.
Under the conditions in
Section~\ref{sec:requirements}, Eq.~\eqref{eq:density} gives the
log probability density of the sampler in Eq.~\eqref{eq:chain} almost
everywhere on $(-1,1)^d$. Algorithm~\ref{alg:forward-sample} computes it
alongside the sample by subtracting the log-determinant at each step.

\begin{table}[H]
\centering
\begin{minipage}[t]{0.96\textwidth}
\vspace{0pt}
\refstepcounter{algorithm}
\label{alg:forward-sample}
\begin{tabular}{@{}p{\linewidth}@{}}
\toprule
\textbf{Algorithm~\thealgorithm} Forward sampling and action probability density\\
\midrule
\begin{minipage}[t]{\linewidth}
\vspace{0pt}
\small
\begin{algorithmic}[1]
\Require observation condition $c$, step count $K$, policy parameters $\theta$
\Ensure normalized action chunk $a$ and $\log\pi_\theta(a\mid c)$
\State $h\gets 1/K$
\State $x_0\sim\mathcal N(0,I_d)$
\State $\ell\gets\log\phi(x_0)$
\For{$k=0,\ldots,K-1$}
  \State $t_k\gets kh$
  \State Apply the composed JTT update to obtain $x_{k+1}=F_{\theta,h}(x_k,t_k,c)$ and its log-determinant $\ell_k$
  \State $\ell\gets\ell-\ell_k$
\EndFor
\State $a\gets\tanh x_K$
\State $\ell\gets\ell-\sum_{j=1}^{d}\log(1-a_j^2)$
\State \Return $a,\ \log\pi_\theta(a\mid c)\gets\ell$
\end{algorithmic}
\end{minipage}\\
\bottomrule
\end{tabular}
\end{minipage}
\end{table}

\subsection{Design Requirements for the Euler-step Network}
\label{sec:requirements}
Section~\ref{sec:finite-map} derives the action probability density
from the determinant of each Euler-step. To compute these determinants
while learning from Flow Matching paths, we impose four requirements
on the Euler-step network.

\begin{enumerate}
\item \textbf{Identity at zero step size and a local velocity.}
The step size $h$ specifies the duration of each transport step.
A zero-length step must therefore leave the state unchanged:
$F_{\theta,0}(x,t,c)=x$. Without this condition, the transformation
could produce a finite displacement even as the time interval tends
to zero. Here the local velocity is induced by the directly
parameterized finite map family. Wherever the derivative with respect
to $h$ exists at zero, the identity condition gives the expansion
\begin{equation}
 v_\theta(x,t,c)=
 \left.\partial_hF_{\theta,h}(x,t,c)\right|_{h=0},\qquad
 F_{\theta,h}(x,t,c)=x+h\,v_\theta(x,t,c)+o(h).
 \label{eq:generator}
\end{equation}
Thus, the displacement divided by $h$ approaches $v_\theta$ as
$h\to0$, connecting the learned finite update to a velocity field
in the zero-step limit.

\item \textbf{Global invertibility and spatial regularity.}
To apply the change-of-variables formula, we require the Euler update,
for each fixed $(t,c,h)$, to be a
piecewise continuously differentiable bijection of $\mathbb R^d$ with
a nonzero spatial Jacobian determinant almost everywhere. Every
output then has a unique preimage, allowing us to propagate its
probability density through the full sequence of updates.

\item \textbf{Tractable determinants during sampling.}
An invertible network does not necessarily provide a tractable
determinant. For a typical neural network, computing the Jacobian determinant
requires additional calculations beyond the forward pass. Our network
must therefore compute the Jacobian determinant efficiently during the forward pass,
without constructing the full Jacobian matrix,
so that $\log|\det J_{F_{\theta,h}}|$ can be accumulated to obtain
the sample's probability density.

\item \textbf{Numerically well-behaved transformations.}
A transformation can be analytically invertible yet sharply compress
or stretch particular directions, making repeated updates and
probability density gradients sensitive to numerical perturbations.
The architecture and implementation should limit such directional
distortions and allow stable evaluation of log-determinants and their
gradients. This is a numerical design goal.
\end{enumerate}

We will describe how JTT addresses these requirements in the next section.

\subsection{Network Architecture}
\label{sec:map}
One JTT feedforward computation consists of $M$ sequential \emph{map heads}. Each head
combines a learned orthogonal coordinate frame, a Transformer
conditioner, and monotone scalar splines. A spline is a one-dimensional
function made of connected curve segments; we use it to transform each
scalar coordinate monotonically. Figure~\ref{fig:network} shows the complete
feedforward calculation (A), the operations within a head (B), and the determinant
calculation (C).

A head first rotates the state $x\in\mathbb R^d$ into its learned
frame, $y=Qx$, then embeds the right-shifted scalar sequence
$(0,y_1,\ldots,y_{d-1})$ as $d$ tokens. Together with time and
observation information, these tokens enter the Transformer, which
predicts spline parameters for each coordinate. The splines transform
the unshifted values $y_i$, and $Q^\top$ rotates the result back to the
original frame. A
\emph{map head} thus denotes a complete invertible transformation,
distinct from the attention heads inside its Transformer.

The Transformer predicts parameters, while the splines perform the
state transformation. The conditioner can therefore use nonlinear
representations of the state and observations, with invertibility
guaranteed by the transformation it parameterizes. Causality specifies
the dependency order among scalar coordinates of the generative state;
all commands in an action chunk share the current observation condition
$c$.

\subsubsection{Causal Scalar Transformations}
For coordinate $i$, the Transformer predicts parameters
$\eta_{\theta,i}(y_{<i},t,c)$ using only the preceding coordinates
$y_{<i}=(y_1,\ldots,y_{i-1})$, time, and observations. The resulting
scalar transformation is
\begin{equation}
 S_{\theta,h}(y)_i
 =s_{\theta,h}\!\left(y_i;
       \eta_{\theta,i}(y_{<i},t,c)\right),\qquad i=1,\ldots,d.
 \label{eq:causal-map}
\end{equation}
For example, in three dimensions, suppressing the fixed conditions
$(t,c,h)$ gives
\[
 \begin{aligned}
  w_1&=s_1(y_1;\eta_1),\\
  w_2&=s_2(y_2;\eta_2(y_1)),\\
  w_3&=s_3(y_3;\eta_3(y_1,y_2)),
 \end{aligned}
 \qquad
 J_S(y)=
 \begin{pmatrix}
  \partial_{y_1}w_1&0&0\\
  \partial_{y_1}w_2&\partial_{y_2}w_2&0\\
  \partial_{y_1}w_3&\partial_{y_2}w_3&\partial_{y_3}w_3
 \end{pmatrix}.
\]
The Jacobian is lower triangular, so its determinant equals the product
of the diagonal scalar derivatives, evaluated with the prefix held
fixed. Since $\eta_{\theta,i}$ does not depend on $y_i$, these entries
require only the derivative of each scalar transformation with respect
to its own input, without differentiating through the conditioner.

Right-shifted tokens and a causal self-attention mask enforce this
prefix dependence: token $i$ can attend to tokens $1$ through $i$,
whose scalar inputs extend only to $y_{i-1}$. Cross-attention supplies
time and observation information. Since all input coordinates are
available during forward sampling, one Transformer pass predicts the
parameters for all coordinates in parallel.

We ensure a nonzero determinant by choosing scalar transformations
with positive derivatives. Each coordinate uses a monotone
rational-quadratic spline (RQS)~\citep{durkan2019neural}, which divides
a fixed interval into bins and maps each input bin to an output bin
with a rational-quadratic curve. The predicted bin widths, bin heights,
and knot slopes control how the spline nonlinearly expands or contracts
different parts of the coordinate range. Positive widths, heights, and slopes
make the transformation strictly increasing within the interval;
identity tails extend it to a bijection of $\mathbb R$. Every scalar
transformation is therefore bijective at a fixed prefix, as is the
full causal transformation. Analytic RQS derivatives give the diagonal
entries needed to evaluate its determinant.

\subsubsection{Step-Size Parameterization and Orthogonal Heads}
To obtain the identity at $h=0$, we parameterize the spline through
deviations from its identity settings, controlled by the step size.
We convert the conditioner's raw outputs $b$ into bounded parameter
rates and scale them by $h$:
\begin{equation}
 r(b)=R\tanh(b/R),\qquad \xi_h=h\,r(b)
 \label{eq:parameter-rate}
\end{equation}
Here $R>0$ controls the rate magnitude, and $\xi_h$ gives the parameter
offsets for the current step. We then convert these offsets into valid bin
widths, bin heights, and knot slopes. Thus $h$ scales the parameter
rates before this conversion; it does not multiply the final widths,
heights, or slopes.

Zero offsets produce an identity spline: the input and output bin
boundaries coincide, and every knot slope equals one. Hence $h=0$
ensures $\xi_h=0$ and $S_{\theta,0}=\mathrm{id}$ for any network
output. For $h>0$, the predicted offsets change the spline's shape.
This parameterization embeds the identity condition from
Section~\ref{sec:requirements} directly in the architecture.
Appendix~\ref{app:causal-proof} gives the RQS parameterization and
its positivity and identity conditions.

Within one causal head, early coordinates cannot depend on later
ones. We therefore compose heads in different learned coordinate
frames, so that subsequent heads act on different linear combinations
of the preceding outputs. Each head takes the form
\begin{equation}
 y=Qx,\qquad H_{\theta,h}(x)=Q^\top S_{\theta,h}(Qx),
 \qquad Q^\top Q=I.
 \label{eq:head}
\end{equation}
Time and observations are held fixed in this expression. Orthogonal
frames preserve lengths, contribute no additional determinant factor,
and can be inverted by taking the transpose. Composing $M$ heads gives
one sampling update:
\[
 F_{\theta,h}=H^{(M)}_{\theta,h}\circ\cdots\circ H^{(1)}_{\theta,h}.
\]
Each head has its own coordinate frame and conditioner. Since the
orthogonal transformations introduce no probability density correction,
the scalar log-derivatives suffice to track the density through the
composition. Every head, and therefore the complete update, becomes
the identity at $h=0$. The same update parameters are used at all $K$
sampling steps, with $t_k$ supplied to each conditioner.

The resulting JTT update is invertible, and its determinant depends
only on scalar derivatives, as formalized below.

\begin{proposition}[Bijective feedforward transformations with tractable determinants]
\label{prop:causal}
For fixed $(t,c,h)$, suppose every scalar transformation in
Eq.~\eqref{eq:causal-map} is a bijection $\mathbb R\to\mathbb R$ for
every fixed prefix. Then each head in Eq.~\eqref{eq:head} is bijective.
At differentiable points with positive scalar derivatives,
\begin{equation}
 \log|\det J_{H_{\theta,h}}(x)|
 =\sum_{i=1}^d
 \log\frac{\partial s_{\theta,h}(y_i;\eta_{\theta,i})}{\partial y_i}.
 \label{eq:head-det}
\end{equation}
Composing heads and sampling updates adds these log-determinants at their
respective input states.
\end{proposition}

The proposition is proved in Appendix~\ref{app:theory}.

\begin{figure}[t]
 \centering
 \includegraphics[width=\textwidth]{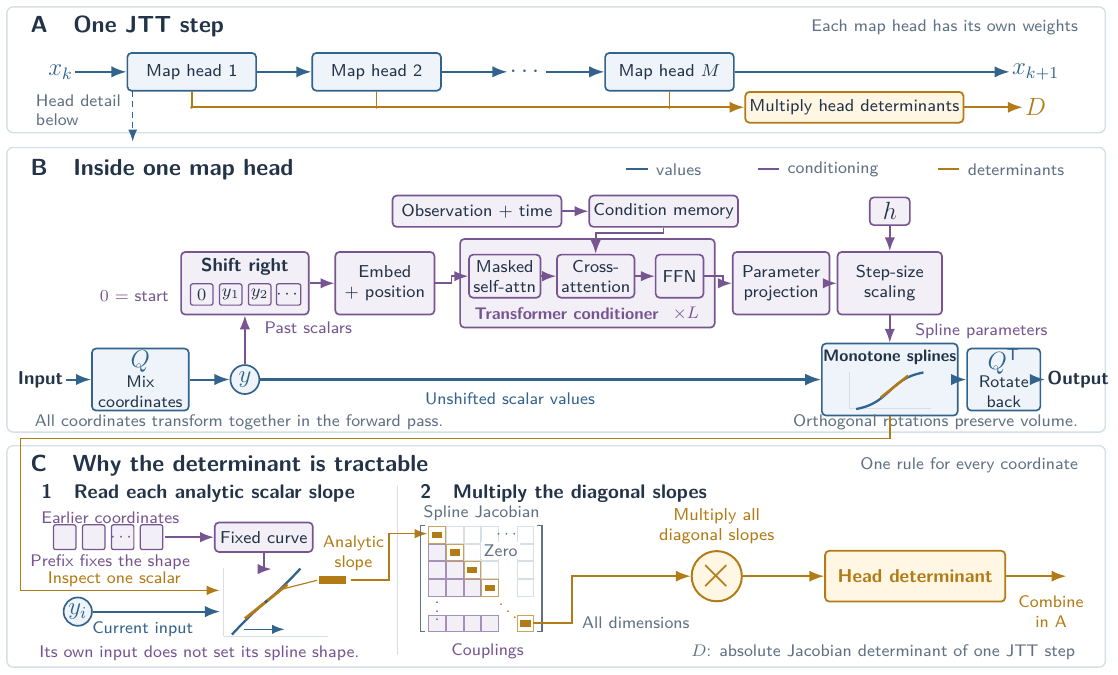}
 \caption{The architecture of Jacobian-Tractable Transformer (JTT).
 \textbf{(A)} A complete update applies $M$ map heads in sequence and
 returns the next noisy sample $x_{k+1}$ and the absolute Jacobian determinant
 $D$.
 \textbf{(B)} A causal Transformer with $L$ decoder blocks uses shifted
 scalar tokens to predict monotone spline parameters. Cross-attention
 provides time and observation information. The projected outputs are
 converted to bounded parameter rates and multiplied by $h$ to obtain
 offsets from the identity settings. These offsets determine the spline
 parameters. Between the two orthogonal rotations, the splines transform
 the unshifted scalar values.
 \textbf{(C)} For a fixed coordinate prefix, time, and observation,
 $y_i$ determines where its spline is evaluated, but not the spline's
 shape. The positive analytic derivatives form the diagonal of the
 lower-triangular spline Jacobian, and their product gives the head
 determinant without evaluating off-diagonal entries. The
 orthogonal rotations preserve its magnitude. At step $k$, summing
 the head log-determinants at their respective inputs gives
 $\ell_k=\log D$.}
 \label{fig:network}
\end{figure}

\subsection{Training}
\label{sec:training}
\subsubsection{Conditional Euler-step Training}
\label{sec:cet}
Standard Flow Matching trains a velocity field along a continuous-time
path, whereas our policy samples through a discrete sequence of finite
maps. The identity condition $F_{\theta,0}(x,t,c)=x$ defines a local velocity,
$\left.\partial_hF_{\theta,h}(x,t,c)\right|_{h=0}$, wherever the
derivative exists (Section~\ref{sec:requirements}). This derivative
describes only the map family's first-order behavior as $h\to0$.

At sampling time, the policy applies the finite map at a fixed positive
step size $h=1/K$. Two map families with the same derivative at $h=0$
can therefore produce different displacements at the sampling step
size: equal zero-step velocities do not imply equal finite updates.

CET therefore supervises the finite displacement at the sampling step
size. Dividing this displacement by $h$ gives a prediction of the
Flow Matching path velocity that requires only one forward map
evaluation. No differentiation with respect to $h$ is needed, and
parameter gradients follow by ordinary backpropagation.
Appendix~\ref{app:cet-proof} explains how regression with this finite
quotient differs from regression with the zero-step generator.

For each training example, we sample a demonstration action chunk and apply
the inverse-tanh transformation described in Section~\ref{sec:policy-coordinates}
to obtain the target $u_D$. We also sample Gaussian noise $z$ and a time
$t_k$ from the sampling grid.
Appendix~\ref{app:protocol} specifies how boundary values are handled.
Interpolating between $z$ and $u_D$ gives the input $X_{t_k}$ on a
straight path with velocity $u_D-z$. Applying the update network at
this state and dividing its displacement by $h$ yields the predicted
path velocity:
\begin{equation}
 X_t=(1-t)z+tu_D,\qquad V=u_D-z,\qquad
 q_{\theta,h}(x,t,c)=\frac{F_{\theta,h}(x,t,c)-x}{h}.
 \label{eq:quotient}
\end{equation}
The CET loss is similar to the Conditional Flow Matching loss
(Eq.~\eqref{eq:cfm-standard}). Using the sampling step size $h=1/K$ and a uniformly drawn grid
index $k\in\{0,\ldots,K-1\}$, CET minimizes
\begin{equation}
 \mathcal L_{\mathrm{CET},h}
 =\mathbb E_{c,u_D,z,k}\frac{1}{d}
   \|q_{\theta,h}(X_{t_k},t_k,c)-V\|_2^2;
 \label{eq:fm}
\end{equation}
for each endpoint pair, $X_{t_k}+hV=X_{t_{k+1}}$, so this loss trains
the map to advance the path sample toward its next grid position.
Both $X_{t_k}$ and the target velocity $V$ are constructed directly
from the endpoints, requiring only one local map evaluation per
training example. At deployment, the policy starts from noise and
applies all $K$ updates in sequence, using each generated state as
the input to the next update.

The training and sampling procedures thus use the same positive step
size and time grid. Algorithm~\ref{alg:cet-training} describes
how to train our model with demonstrations.

\begin{table}[htbp]
\centering
\begin{minipage}[t]{0.96\textwidth}
\vspace{0pt}
\refstepcounter{algorithm}
\label{alg:cet-training}
\begin{tabular}{@{}p{\linewidth}@{}}
\toprule
\textbf{Algorithm~\thealgorithm} Conditional Euler-step Training\\
\midrule
\begin{minipage}[t]{\linewidth}
\vspace{0pt}
\small
\begin{algorithmic}[1]
\Require demonstration windows after normalization and inverse-tanh transformation $\mathcal D_u=\{(c,u_D)\}$,
  initialized $\theta$, step count $K$, batch size $B$, update budget $U$
\Ensure trained policy parameters $\theta$
\State $h\gets 1/K$
\For{$r=1,\ldots,U$}
  \State Sample a minibatch $\{(c_i,u_{D,i})\}_{i=1}^{B}$ from $\mathcal D_u$
  \For{$i=1,\ldots,B$}
    \State Draw $z_i\sim\mathcal N(0,I_d)$ and $k_i\sim\mathrm{Unif}\{0,\ldots,K-1\}$
    \State $t_i\gets k_i h$; $X_i\gets(1-t_i)z_i+t_i u_{D,i}$; $V_i\gets u_{D,i}-z_i$
    \State $q_i\gets\bigl(F_{\theta,h}(X_i,t_i,c_i)-X_i\bigr)/h$
  \EndFor
  \State $\mathcal L\gets\frac{1}{Bd}\sum_{i=1}^{B}\|q_i-V_i\|_2^2$
  \State Take one optimizer step on $\theta$ using $\nabla_\theta\mathcal L$
\EndFor
\State \Return $\theta$
\end{algorithmic}
\end{minipage}\\
\bottomrule
\end{tabular}
\end{minipage}
\end{table}

Different noise--demonstration pairs may pass through the same state
with different velocities. Given the state, time, and observation
condition, the MSE loss targets the conditional average
over velocities of endpoint pairs consistent with that input. The
following proposition shows that, within our finite-map family, this
conditional supervision and supervision using the marginal velocity have
the same minimizers.

\begin{proposition}[Minimization equivalence of CET and marginal Flow Matching]
\label{prop:cet}
Fix $h=1/K>0$ and a parameter domain $\Theta$. Under the training
distribution in Eq.~\eqref{eq:fm}, let $S=(k,X_{t_k},c)$ denote the
information available to the network, and define
$U=\mathbb E[V\mid S]$ and
$q_{\theta,h}=q_{\theta,h}(X_{t_k},t_k,c)$.
Assume that the distribution of $(S,V)$ is independent of $\theta$ and, for every
$\theta\in\Theta$, $V,q_{\theta,h}\in L^2$ and $q_{\theta,h}$ is
measurable with respect to $S$. Define
\begin{equation}
 \mathcal L_{\mathrm{FM},h}(\theta)
 =\frac{1}{d}\mathbb E\|q_{\theta,h}-U\|^2,\qquad
 C=\frac{1}{d}\mathbb E\|V-U\|^2.
 \label{eq:marginal-risk}
\end{equation}
Then, for every $\theta\in\Theta$,
\begin{equation}
 \mathcal L_{\mathrm{CET},h}(\theta)
 =C+\mathcal L_{\mathrm{FM},h}(\theta),
 \label{eq:regression}
\end{equation}
where $C$ is independent of $\theta$. For the marginal Euler target
$T_h(S)=X_{t_k}+hU$, also define
\begin{equation}
 \mathcal L_{\mathrm{E},h}(\theta)
 =\frac{1}{d}\mathbb E
 \|F_{\theta,h}(X_{t_k},t_k,c)-T_h(S)\|^2
 =h^2\mathcal L_{\mathrm{FM},h}(\theta).
 \label{eq:endpoint-risk}
\end{equation}
Consequently,
\begin{equation}
 \operatorname*{argmin}_{\theta\in\Theta}\mathcal L_{\mathrm{CET},h}
 =\operatorname*{argmin}_{\theta\in\Theta}\mathcal L_{\mathrm{FM},h}
 =\operatorname*{argmin}_{\theta\in\Theta}\mathcal L_{\mathrm{E},h}.
 \label{eq:cet-argmin}
\end{equation}
\end{proposition}

Similar to Conditional Flow Matching, CET uses velocities from directly
sampled endpoint pairs without explicitly computing the marginal
velocity $U$~\citep{lipman2024guide}. See Appendix~\ref{app:cet-proof}
for the proof.

\subsubsection{SAC-Style Policy Learning}
\label{sec:sac}
We follow an offline-to-online actor--critic schedule based on
SAC-Flow~\citep{yu2025sacflow}. Each offline iteration performs two successive
updates on the same batch. The first trains the actor with CET and
the critics using the SAC critic loss. The second applies a
SAC-style actor--critic update to the same batch. During online
learning, the policy executes sampled action chunks and stores the
resulting transitions in replay. After warm-up, it trains the actor
and critics on replay windows using the objectives below. CET is
used only offline; online actor updates use the SAC-style objective.

Let $a_B$ denote the action chunk in a replay window, drawn from
either the offline dataset or online interaction. The SAC-style actor
objective is
\begin{equation}
 \mathcal L_{\mathrm{actor}}=
 \mathbb E_{c\sim\mathcal B,\,a\sim\pi_\theta(\cdot\mid c)}
   [\alpha\log\pi_\theta(a\mid c)-\bar Q_\psi(c,a)]
 +\lambda\mathbb E_{(c,a_B)\sim\mathcal B,\,a\sim\pi_\theta(\cdot\mid c)}
   \frac{1}{d}\|a-a_B\|_2^2,
 \label{eq:actor}
\end{equation}
where $\bar Q_\psi$ averages the two critics $Q_{\psi_1}$ and
$Q_{\psi_2}$, and $\lambda$ weights the action-matching penalty.

These objectives act on different quantities. CET uses demonstrations
to supervise local transport at noisy intermediate states. The Q term
evaluates the complete sampled chunk in terms of task reward, while
the entropy term uses its log probability. Action matching penalizes
deviations from replay actions directly in action space, constraining
reward-driven updates even when no CET path is sampled. During online
learning, the reference actions change as the policy collects new
replay data. Algorithm~\ref{alg:forward-sample} computes $a$ and its
log probability $\log\pi_\theta(a\mid c)$ together, allowing policy
gradients to backpropagate.

For a complete predicted chunk executed from environment time $\tau$
with $T_p=T_a$,
define its discounted return as
$R_\tau^{(T_a)}=\sum_{i=0}^{T_a-1}\gamma^i r_{\tau+i}$, where
$r_{\tau+i}$ is the reward at each step and $\gamma\in[0,1)$ is the
discount factor. Let $m_\tau^{(T_a)}$ indicate whether bootstrapping
is permitted. Sampling the next chunk
$a'\sim\pi_\theta(\cdot\mid c_{\tau+T_a})$ gives the soft target
\begin{equation}
 y_\tau=R_\tau^{(T_a)}+\gamma^{T_a}m_\tau^{(T_a)}
 \left[\bar Q_{\mathrm{target}}(c_{\tau+T_a},a')
       -\alpha\log\pi_\theta(a'\mid c_{\tau+T_a})\right],
 \quad a'\sim\pi_\theta(\cdot\mid c_{\tau+T_a}).
 \label{eq:target}
\end{equation}
Here $\bar Q_{\mathrm{target}}$ averages the two target critics. Each
critic minimizes the squared Bellman residual between its value for
the replayed chunk and $y_\tau$. Both the actor objective and the critic
target use the log probability of the complete executed chunk.
Appendix~\ref{app:protocol} specifies the settings of the replay masks,
optimizers, target-network updates, and initial entropy temperature.

\section{Evaluation}
\label{sec:experiments}

Our experiments address three questions: can JTT learn effective
imitation policies under its structural constraints; can the same
policy support offline-to-online actor--critic learning; and how
do finite-step supervision and learned coordinate frames affect task
performance? Lift evaluates imitation learning and the two design
choices, while Cube-Triple evaluates online learning and compares
offline training configurations.

\subsection{Experimental Setup}
\label{sec:exp-protocol}
\paragraph{Benchmarks.}
For Robomimic Lift~\citep{mandlekar2022robomimic}, we use the
low-dimensional Multi-Human demonstration dataset, comprising 27,582
training windows. The robot must grasp and lift a cube based on a
19-dimensional state observation, with each action chunk containing
four seven-dimensional end-effector and gripper commands ($d=28$).
OGBench Cube-Triple Task~4~\citep{park2025ogbench} uses the
\texttt{cube-triple-\allowbreak play-v0} dataset and requires the
robot to move three cubes to a target configuration. The policy
maps a 46-dimensional state observation to a chunk of five commands,
each with five control coordinates. Both tasks use $T_o=1$ and
$T_p=T_a$: the robot executes the full predicted chunk before replanning.

\paragraph{Policies and training.}
PrT-Flow applies four orthogonal JTT heads per update, starting from
a Gaussian base distribution and using a final tanh transformation
to map outputs to the normalized action bounds. On Lift, we use
16 sampling steps with CET and evaluate the exponential moving
average (EMA) weights after 300 training epochs. On Cube-Triple,
we use four sampling steps and train the randomly initialized policy
and critics for 30,000 paired offline iterations, followed by
600,000 primitive environment steps of online learning. The main
configuration uses finite-step CET at the four sampling grid times
and an online action-matching weight of $\lambda=100$.
Appendix~\ref{app:protocol} gives the hyperparameters and action
preprocessing details.

\paragraph{Baselines and metrics.}
The Lift reference is a denoising diffusion probabilistic model
(DDPM) policy trained on the same data, with the same action horizons
and evaluation episodes. On Cube-Triple, we compare against Flow-G,
the gated velocity policy from the SAC-Flow
family~\citep{yu2025sacflow}, with sampling noise standard deviations
$\sigma=0.01$ and $0.05$. All three runs use 30k paired offline
iterations and 600k online primitive steps. Flow-G samples with a
gated MLP and added noise, whereas PrT-Flow uses deterministic JTT
updates; all three use the same raw demonstration corpus. Success
rate is the fraction of episodes satisfying the environment's success
criterion, and mean return is the average undiscounted episode reward.

Each evaluation comprises 50 episodes, with a maximum of 500 steps
per episode on Lift and 1,000 on Cube-Triple. We report Lift results
after 300 training epochs and evaluate Cube-Triple every 10k online
steps, reporting results at 600k steps.

\subsection{Main Results}
\label{sec:main-results}
\paragraph{Imitation learning on Lift.}
\label{sec:lift-results}
PrT-Flow achieves a 96\% success rate, matching the DDPM baseline when
both policies are evaluated after 300 training epochs (Table~\ref{tab:lift-il}).
This result shows that causal, invertible JTT updates can support
effective closed-loop manipulation despite their structural constraints.

\begin{table}[H]
\centering
\caption{Lift imitation learning after 300 epochs. Both policies are
evaluated on the same set of 50 episodes.}
\label{tab:lift-il}
\begin{tabular}{@{}lcc@{}}\toprule
Policy & Evaluation epoch & Success rate\\\midrule
DDPM reference & 300 & 96\%\\
PrT-Flow (JTT, CET) & 300 & 96\%\\\bottomrule
\end{tabular}
\end{table}

\paragraph{Offline-to-online learning on Cube-Triple.}
\label{sec:cube-results}
At 600k online steps, PrT-Flow succeeds in 47 of 50 development
evaluation episodes, compared with 44 for Flow-G ($\sigma=0.01$)
and 28 for Flow-G ($\sigma=0.05$). Their mean returns are $-578.28$,
$-613.80$, and $-1561.86$, respectively (Table~\ref{tab:cube-main}).
Figure~\ref{fig:cube-learning-curves} shows the full learning
trajectories on this shared development set. In separate final
evaluations after training, PrT-Flow succeeds in 48 of 50 episodes,
compared with 46 for Flow-G ($\sigma=0.01$) and 25 for Flow-G
($\sigma=0.05$).

Although Flow-G ($\sigma=0.01$) achieves a slightly higher final mean
return ($-475.48$ versus $-530.80$), PrT-Flow improves more rapidly
during online training (Figure~\ref{fig:cube-learning-curves}). At
half the online training budget (300k primitive environment steps),
PrT-Flow achieves a development success rate of 68\% and a mean return
of $-1158.22$, compared with 38\% and $-2035.18$ for Flow-G
($\sigma=0.01$). These learning curves indicate an advantage in
online sample efficiency for PrT-Flow in this comparison.
The resulting policy achieves effective online control using JTT
and the action probabilities computed during its forward pass.
Appendix~\ref{app:verification} reports numerical checks of the
trained policy's forward log probabilities and their parameter gradients.

\begin{table}[H]
\centering\small
\caption{Cube-Triple Task~4 after 30k paired offline iterations and
600k online primitive environment steps. Development results use
the same set of 50 evaluation episodes for all policies. Final
results come from separate 50-episode evaluations after training.}
\label{tab:cube-main}
\begin{tabularx}{\textwidth}{@{}Xrrrr@{}}\toprule
Policy & Dev. SR & Dev. return & Final SR & Final return\\\midrule
Flow-G, $\sigma=0.01$ & 88\% & $-613.80$ & 92\% & $-475.48$\\
Flow-G, $\sigma=0.05$ & 56\% & $-1561.86$ & 50\% & $-1635.42$\\
PrT-Flow, CET & 94\% & $-578.28$ & 96\% & $-530.80$\\\bottomrule
\end{tabularx}
\end{table}

\begin{figure}[t]
\centering
\includegraphics[width=\textwidth]{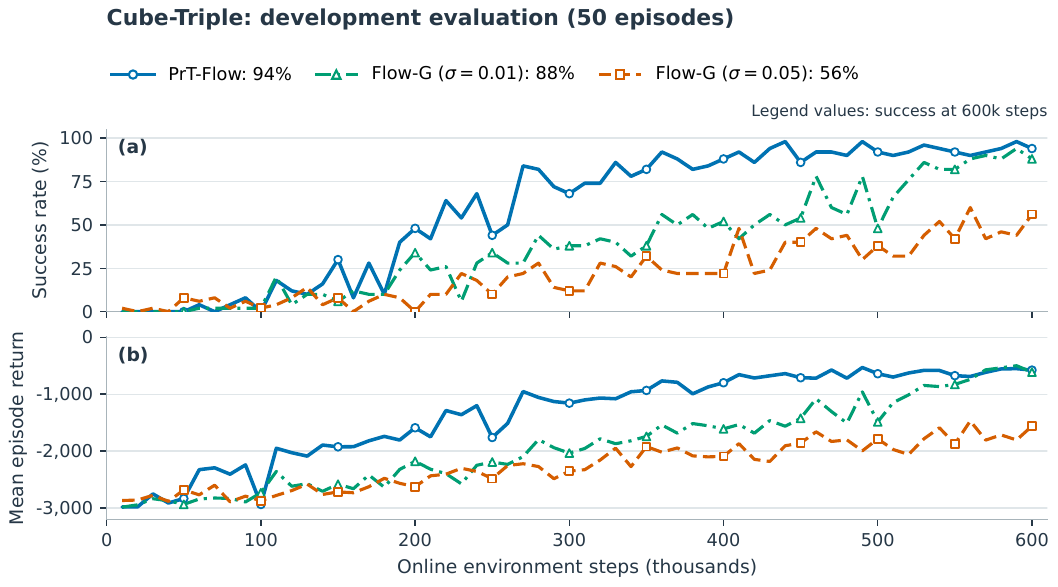}
\caption{Learning curves on Cube-Triple for PrT-Flow and Flow-G
($\sigma=0.01$ and $0.05$). Panels \textbf{A} and \textbf{B} show
success rate and mean return on the shared set of 50 development
episodes. Lines connect evaluations without smoothing. At 600k
online steps, success rates are 94\% for PrT-Flow, 88\% for Flow-G
($\sigma=0.01$), and 56\% for Flow-G ($\sigma=0.05$). Values in the
legend correspond to this endpoint, rather than each run's best
checkpoint.}
\label{fig:cube-learning-curves}
\end{figure}

\subsection{Ablations and Offline Training Configurations}
\label{sec:ablation}
We first examine two design choices using pairs of Lift runs with
matched training settings: the supervised quantity and the orthogonal
coordinate frames.
Both comparisons use 16 sampling steps and evaluate EMA policies
after 300 training epochs (Table~\ref{tab:lift-ablations}).
These are separate runs with a 300-epoch cosine schedule; the
PrT-Flow benchmark in Table~\ref{tab:lift-il} uses the epoch-300
checkpoint of a 3,000-epoch schedule.

\begin{table}[H]
\centering\small
\caption{Design ablations on Lift with $K=16$ after 300 training epochs.
Each block reports a separate pair of runs with matched training settings,
with 50 evaluation episodes per model. Within each pair, only the design choice listed
in the first column changes.}
\label{tab:lift-ablations}
\begin{tabularx}{\textwidth}{@{}lXr@{}}\toprule
Design choice & Configuration & Successful episodes\\\midrule
Supervised quantity & CET finite-step quotient & \textbf{49/50}\\
 & Zero-step generator & 33/50\\\midrule
Coordinate frame & Learned orthogonal frames & \textbf{49/50}\\
 & Fixed initial random orthogonal frames & 33/50\\\bottomrule
\end{tabularx}
\end{table}

\paragraph{Finite-step supervision.}
The two policies share the network architecture, initialization,
training time grid, and demonstration endpoints in generative coordinates.
We change only whether the training objective regresses the finite-step
quotient $(F_{\theta,h}-\mathrm{id})/h$ at $h=1/16$ or the zero-step
generator $\partial_hF_{\theta,h}|_{h=0}$ against the conditional path
velocity. CET succeeds in 49 of 50 episodes, compared with 33 for
generator supervision. This result supports supervising the finite
update used during sampling.

\paragraph{Learned coordinate frames.}
In a separate matched pair, both policies use CET and start from the
same initial weights and random orthogonal frames. One policy learns
the frames jointly with the rest of JTT; the other keeps them fixed
at initialization. The learned-frame policy succeeds in 49 of 50
episodes, compared with 33 for fixed frames. Thus, learning the
coordinate system improves task performance in this comparison,
supporting the use of trainable orthogonal transformations alongside
the causal scalar maps.

\paragraph{Offline training configurations on Cube-Triple.}
We also train the same JTT architecture with an alternative offline
velocity objective, regressing the zero-step generator
$\partial_hF_{\theta,h}|_{h=0}$ against the conditional path velocity
at continuously sampled times, with endpoints in normalized action
coordinates. CET instead supervises the finite displacement divided
by $h$ at $h=1/4$, using the four sampling times and demonstration
endpoints transformed as in Section~\ref{sec:policy-coordinates}.
Both configurations retain the same paired offline actor--critic
schedule and SAC-style online learning. They thus differ in how
the local maps are supervised offline, while preserving invertible
sampling and tractable probability density evaluation.

\begin{table}[H]
\centering\small
\caption{PrT-Flow with two offline training configurations, evaluated
after 30k paired offline iterations and 600k online steps. Both use
the same development and final evaluation sets, with 50 episodes
in each set.}
\label{tab:cube-joint}
\begin{tabularx}{\textwidth}{@{}Xlrrrr@{}}\toprule
Supervised quantity / time sampling & \shortstack[l]{Endpoint\\coordinates} & Dev. SR & Dev. return & Final SR & Final return\\\midrule
Generator velocity / continuous & Action & 88\% & $-806.02$ & 94\% & $-619.72$\\
CET quotient / grid & Generative space & 94\% & $-578.28$ & 96\% & $-530.80$\\\bottomrule
\end{tabularx}
\end{table}

CET achieves higher development success at 40 of the 42 checkpoints
from 190k to 600k steps, with a final success rate of 96\% versus
94\% for velocity supervision. Although the final difference is
one episode, the checkpoint comparison indicates different training
trajectories. Because the estimator, time sampling, and target
coordinates all change, this comparison evaluates the complete
offline configurations.

\subsection{Discussion}
\label{sec:discussion}
The analysis shows how PrT-Flow computes the action probabilities
defined by the finite-step sampling procedure used at deployment.
JTT provides analytic
Jacobian corrections through invertible finite maps, while CET
supervises finite displacements at the deployment step size and
relates this objective to marginal Flow Matching regression over
the same map family. The experiments show that the invertible and
causal structure needed for this computation still permits effective
robot behavior: PrT-Flow achieves high imitation success on Lift
and supports offline-to-online actor--critic learning on Cube-Triple
with the same action generation and probability evaluation procedure.
Thus, one policy can combine multi-step Flow Matching generation
with an explicit probability density for the actions it executes.

PrT-Flow and methods such as SAC-Flow differ in how they obtain
policy probabilities. SAC-Flow introduces random perturbations
into finite-step sampling to define probabilities over sampling
paths. PrT-Flow applies deterministic, invertible maps after sampling
the initial Gaussian noise, computing the probability density of the
final action chunk directly using the change-of-variables formula. In our Flow-G
experiments (Figure~\ref{fig:cube-learning-curves}), changing only
the sampling noise standard deviation from $\sigma=0.01$ to $0.05$
produces markedly different learning trajectories and fixed-budget
performance, indicating that the noise scale itself can substantially
affect training. PrT-Flow does not need stepwise noise to compute
action probabilities, so it requires no additional sampling noise
scale $\sigma$ for this purpose. Its high online control performance
on Cube-Triple shows that explicit finite-step action densities can
support effective policy learning without a stochastic path-based
probability construction.

This finite-step probability computation also provides a basis for
extensions beyond SAC's action-entropy objective. Inverting the
transformations allows the same policy to re-evaluate a given
action under updated parameters, making it naturally applicable to
likelihood ratios, KL regularization, importance weighting, and
replay-action rescoring. Extensions to longer prediction horizons,
probability densities of executed action prefixes, visually conditioned
policies, and real robots would help establish the scope of
probability-tractable Flow Matching as a general generative policy
representation.

\section{Conclusion}

We introduced PrT-Flow to make the probability density of a robot's
executed action chunk directly computable while preserving multimodal multi-step
Flow Matching generation. The central idea is to make the finite
sampling procedure itself invertible with tractable Jacobian determinants,
so that action probabilities can be computed directly from the sampling steps. JTT combines causal monotone splines and orthogonal
coordinate transformations to accumulate density changes in the
same forward pass that generates an action; CET supervises finite
displacements at the deployment step size, aligning training with
the sampling maps actually used. Proposition~\ref{prop:cet}
establishes that, at a fixed step size and within the same finite
map family, CET and marginal Flow Matching regression have the
same set of minimizing parameters.

The experiments show that the structural constraints needed to
track probabilities do not prevent effective robot learning:
PrT-Flow achieves 48 successes in 50 imitation-learning evaluation
episodes on Robomimic Lift and serves as the action generator in
offline-to-online SAC-style learning on OGBench Cube-Triple, also
reaching 48 successes in 50 final evaluation episodes. These results
provide evidence that a Flow Matching policy can retain multi-step
generation while explicitly evaluating the probability density of
the actions it executes. Constraining the finite sampling maps
allows a single policy to provide both capabilities.

PrT-Flow thus provides a broader basis for studying generative robot
policies: the same finite-step model generates actions from noise
and returns their explicit probability densities. Future work can
exploit this property across a wider range of policy optimization
methods. Finite-step generative policies with tractable probabilities
may thereby offer a general policy representation that connects
the expressive capacity of Flow Matching with likelihood-based
reinforcement learning.

\section*{AI Use Statement}

ChatGPT and OpenAI Codex were used to assist with Python code implementation, Lean 4 formalization, English-language editing, and LaTeX formatting during the preparation of this manuscript.

\bibliographystyle{plainnat}
\bibliography{references}

\newpage
\appendix
\small
\section{Experimental Details}
\label{app:protocol}
\begin{table}[H]
\centering\small
\caption{Optimization settings for the main PrT-Flow experiments.
Cube-Triple uses Adam for the actor, critics, and entropy temperature.}
\label{tab:hyperparameters}
\begin{tabularx}{\textwidth}{@{}p{0.45\textwidth}X X@{}}
\toprule
Parameter & Lift & Cube-Triple\\
\midrule
Batch size & 256 & 256\\
Optimizer & AdamW & Adam\\
Learning rate & $10^{-4}$ & $3\times10^{-4}$\\
Weight decay & $10^{-6}$ & ---\\
AdamW $(\beta_1,\beta_2)$ / $\epsilon$ & $(0.9,0.999)$ / $10^{-8}$ & ---\\
Cosine warm-up (optimizer steps) & 100 & ---\\
EMA power / maximum decay & $0.75$ / $0.995$ & ---\\
Critic hidden layers / width & --- & 4 / 512\\
Critic aggregation & --- & Mean of twin Q-values\\
Discount $\gamma$ / target coefficient $\tau_{\mathrm{target}}$ & --- & $0.99$ / $0.005$\\
Replay buffer-size option & --- & $2\times10^6$\\
Online warm-up (environment steps) / UTD & --- & 5,000 / 1\\
Initial entropy temperature / target entropy & --- & $0.2$ / 0\\
Online action-MSE weight & --- & 100\\
\bottomrule
\end{tabularx}
\end{table}

\paragraph{Preprocessing.}
Lift observations use per-coordinate min--max normalization from training
statistics to a nominal $[-1,1]$ range without clipping; $10\%$ of
demonstrations are reserved for validation. Cube-Triple's offline
actions are clipped with margin $10^{-5}$. After task preprocessing,
any remaining action coordinates at $\pm1$ are replaced by
$\pm(1-10^{-6})$, leaving interior values unchanged, to obtain
$\tilde a_D$; CET then uses $u_D=\operatorname{atanh}\tilde a_D$.

\paragraph{Training details.}
The main Lift result is the epoch-300 checkpoint of a 3,000-epoch cosine
schedule, whereas the matched ablations use a full 300-epoch schedule;
both use the warm-up and EMA settings in Table~\ref{tab:hyperparameters}.
For incomplete replay windows, invalid positions are masked, but the loss
is divided by all entries in the batch--time--coordinate array, including
masked entries. The alternative generator-supervised configuration in
Table~\ref{tab:cube-joint} uses continuous times and the latent endpoint
$a_D$, corresponding to $\tanh(a_D)$ in action space, whereas CET uses
the inverse-tanh endpoint.

\subsection{Flow-G Baseline}
\label{app:flowg}
The two Flow-G configurations differ only in sampling noise
($\sigma=0.01$ or $0.05$); both use transition noise $0.05$ in SAC actor
training. They retain the noise reuse and log-probability calculation of
the authors' released offline-to-online implementation.\footnote{\href{https://github.com/Elessar123/SAC-FLOW/blob/3df8b9dc33de5160851a82322bc6c9005886f3ce/offline-to-online/agents/acfql_gru_ablation_online_sac.py\#L409-L513}{Official Flow-G sampling and log-probability routines}, revision \texttt{3df8b9d}.}
SAC-Flow's Gaussian construction defines a joint path probability density,
with action probability density obtained by integrating out intermediate
states~\citep[Appendix~A.2, Eqs.~(15)--(16)]{yu2025sacflow}.
The released implementation instead reuses fixed-key perturbations across
Euler steps while summing Gaussian transition log terms, the base log
probability, and the tanh correction. This sum serves as its entropy
surrogate rather than the exact log probability of the sampled path or
action. All three policies share the development evaluation panel and
undergo separate 50-episode final evaluations.

\section{Numerical Verification}
\label{app:verification}

\paragraph{Scalar-map checks.}
We test the policy's scalar RQS configuration (eight bins, $B=4$, rate
bound $16$) at $h=1/16$ and $h=1/4$. For each of five random seeds, we
draw independent sets of 4,096 inputs and outputs, including 128 points
in the identity tails in each set, and independent standard-normal raw
parameters. With parameters held fixed, we compare analytic and
autodifferentiated log-derivatives, check both inverse compositions, and
compare forward and inverse log probabilities under a standard Gaussian
base. The same draws are evaluated on CPU in FP32 and FP64, giving
40,960 checks per direction and precision.
Table~\ref{tab:scalar-consistency} reports the maximum absolute
discrepancies. All evaluated derivatives are positive; the zero-step
and identity-tail checks also pass.

\begin{table}[H]
\centering\small
\caption{Numerical consistency of scalar RQS maps. Here $f$ denotes a
map with fixed parameters and step size. Each entry is the maximum
absolute discrepancy across the five seeds and two step sizes; the
last two rows also take the maximum over both directions.}
\label{tab:scalar-consistency}
\begin{tabularx}{\textwidth}{@{}Xrr@{}}\toprule
Comparison & FP32 & FP64\\\midrule
$f^{-1}(f(x))$ vs.\ $x$
 & $9.54\times10^{-7}$ & $1.33\times10^{-15}$\\
$f(f^{-1}(y))$ vs.\ $y$
 & $7.15\times10^{-7}$ & $1.78\times10^{-15}$\\
Analytic vs.\ autodifferentiated log-derivative
 & $8.20\times10^{-7}$ & $1.55\times10^{-15}$\\
Forward vs.\ inverse log probability
 & $4.53\times10^{-6}$ & $6.22\times10^{-15}$\\\bottomrule
\end{tabularx}
\end{table}

\paragraph{Trained-policy checks.}
We check probability density evaluation on the main Cube-Triple CET
policy at 600k online steps, which generates 25-dimensional action
chunks through four finite updates. On the CPU, we compare FP32 and
FP64 forward log probabilities using the same parameters, 32 real
dataset states, and fixed Gaussian noise. For the concatenated
parameter gradient $g$ of the mean log probability, we report
$\|g_{32}-g_{64}\|_2/\|g_{64}\|_2$.

For three of these sampling trajectories, we also compare each
update's analytic log-determinant with that of its FP64
autodifferentiated Jacobian, giving 12 local comparisons.
Table~\ref{tab:numerical-consistency} shows close numerical agreement
for both the local determinants and the forward probability computations.

\begin{table}[H]
\centering\small
\caption{Numerical consistency of the main Cube-Triple CET policy
at 600k online steps. FP32/FP64 comparisons use 32 real states with
fixed noise; local Jacobian checks use 12 updates across three
four-step sampling trajectories.}
\label{tab:numerical-consistency}
\begin{tabularx}{\textwidth}{@{}Xlr@{}}\toprule
Comparison & Metric & Discrepancy\\\midrule
Analytic vs. autodifferentiated local log-determinant
 & Maximum absolute difference & $\leq 9.2\times10^{-12}$\\
FP32 vs. FP64 forward log probability
 & Maximum absolute difference & $1.20\times10^{-4}$\\
FP32 vs. FP64 parameter gradient of mean log probability
 & Relative $L_2$ difference & $9.77\times10^{-7}$\\\bottomrule
\end{tabularx}
\end{table}

\section{Mathematical Guarantees}
\label{app:theory}

This appendix proves the inversion and Jacobian identities for
computing the policy's probability density and the minimization
equivalence underlying CET. It also distinguishes regression at a
fixed step size from zero-step generator training.
Propositions~\ref{prop:causal} and~\ref{prop:cet} have both been
formally verified in Lean~4.

\subsection{Causal Inversion and Scalar RQS Derivatives}
\label{app:causal-proof}

\paragraph{RQS parameterization.}
Each head learns an unconstrained matrix $A$ and sets
$Q=\exp(A-A^\top)\in\mathrm{SO}(d)$. Using the step-scaled offsets
$\xi_h=h\,r(b)$ from Eq.~\eqref{eq:parameter-rate}, the widths and
heights of $n=n_{\mathrm{bin}}$ bins on $[-B,B]$ are
\begin{equation}
\begin{aligned}
 w_j&=\varepsilon_w+(2B-n\varepsilon_w)\operatorname{softmax}(\xi_h^w)_j,\\
 q_j&=\varepsilon_h+(2B-n\varepsilon_h)\operatorname{softmax}(\xi_h^h)_j,
 \qquad j=1,\ldots,n.
\end{aligned}
\label{eq:rqs-parameters}
\end{equation}
For $j=0,\ldots,n$, the knot slopes are
\begin{equation}
 m_j=\exp\!\left\{\operatorname{clip}\!\left[
 \log\!\left(\varepsilon_d+\operatorname{softplus}\!\left(
 \xi_{h,j}^r+\operatorname{softplus}^{-1}(1-\varepsilon_d)
 \right)\right),-L,L\right]\right\}.
\label{eq:rqs-slopes}
\end{equation}
Here $B>0$, $0<\varepsilon_w,\varepsilon_h<2B/n$,
$0<\varepsilon_d<1$, and $L\geq0$. Cumulative widths and heights
define knots with endpoints $\pm B$; we use the standard RQS
formula~\citep{durkan2019neural} within bins and identity tails.
At $h=0$, $w_j=q_j=2B/n$ and $m_j=1$, so the scalar map and its
orthogonal compositions are the identity.

\begin{proof}[Proof of Proposition~\ref{prop:causal}]

Given an output $w$, invert the first scalar map to recover $y_1$.
Once $y_{<i}$ is known, the conditioner fixes the parameters of the
$i$th scalar map, whose inverse uniquely recovers $y_i$. By induction,
every $w$ has exactly one input, proving surjectivity and injectivity.

At any differentiable point, changing $y_j$ for $j>i$ leaves the
$i$th output unchanged, so the Jacobian of $S$ is lower triangular.
Its diagonal entries are the scalar derivatives with the prefix
held fixed, giving $\det J_S(y)=\prod_i\partial_{y_i}s_i$.
For $H=Q^\top S Q$, the chain rule yields
$J_H(x)=Q^\top J_S(Qx)Q$. Orthogonality of $Q$ then gives
$|\det J_H(x)|=|\det J_S(Qx)|$.

The scalar derivatives are positive, so taking logarithms is valid.
Composition adds the log-determinants of successive heads and steps,
each evaluated at its own input. The proof requires neither a
triangular Jacobian after rotation nor an exact ODE flow.
\end{proof}

For a scalar bin of width $w>0$, height $q>0$, and
endpoint slopes $d_0,d_1>0$, let $\sigma=q/w$ and
$\tau=(x-x_0)/w\in[0,1]$. The RQS derivative is
\begin{equation}
 S'(x)=
 \frac{\sigma^2[d_1\tau^2+2\sigma\tau(1-\tau)+d_0(1-\tau)^2]}
 {[\sigma+(d_0+d_1-2\sigma)\tau(1-\tau)]^2}>0.
 \label{eq:rqs-derivative}
\end{equation}

Before squaring, the denominator is
$\sigma[\tau^2+(1-\tau)^2]+(d_0+d_1)\tau(1-\tau)$, which is positive
throughout the bin, as is the numerator. Each bin thus maps its input
interval continuously and strictly increasingly onto the prescribed
output interval and has a unique within-bin inverse. Equal widths
and heights, unit slopes, and coincident input and output knots give
the identity. Adjacent bins share knots and derivatives, and identity
tails preserve bijectivity on the real line. The learned endpoint
slopes may differ from the tails' unit slope, so we use
almost-everywhere derivatives without assuming global $C^1$ regularity.

\subsection{CET Minimizers and Comparison with Generator Training}
\label{app:cet-proof}
\begin{proof}[Proof of Proposition~\ref{prop:cet}]

Let $q=q_{\theta,h}(X_{t_k},t_k,c)$ and $U=\mathbb E[V\mid S]$.
Conditional expectation preserves square integrability, and the
Cauchy--Schwarz inequality ensures that all terms below are
integrable. Since $q-U$ is $S$-measurable and
$\mathbb E[V-U\mid S]=0$, we have

\[
 \mathbb E\langle q-U,V-U\rangle
 =\mathbb E\langle q-U,\mathbb E[V-U\mid S]\rangle=0.
\]

Expanding the squared norm of $q-V=(q-U)-(V-U)$ and dividing by
$d$ gives Eq.~\eqref{eq:regression}. The additive constant is the
same for every $\theta\in\Theta$. Subtracting it preserves every
pairwise loss comparison and hence the full sets of minimizers,
including when those sets are empty. Finally,

\[
 F_{\theta,h}(X_{t_k},t_k,c)-T_h(S)
 =h\bigl(q_{\theta,h}-U\bigr).
\]

Taking squared norms and expectations, then dividing by $d$, gives
Eq.~\eqref{eq:endpoint-risk}. The factor $h^2$ is positive and
independent of $\theta$, which proves Eq.~\eqref{eq:cet-argmin}.
\end{proof}

If $F_{\theta,0}(x,t,c)=x$ and $h\mapsto F_{\theta,h}(x,t,c)$ is
differentiable at zero, the definition of derivative gives
\begin{equation}
 F_{\theta,h}(x,t,c)=x+hv_\theta(x,t,c)+o(h),\qquad
 q_{\theta,h}(x,t,c)\longrightarrow v_\theta(x,t,c).
 \label{eq:generator-limit}
\end{equation}

For jointly differentiable heads that equal the identity at $h=0$,
the generator of their finite composition is the sum of their
generators. Conjugating a head by a fixed linear map conjugates
its generator in the same way. Neither statement establishes a
semigroup law or implies that the finite map is the exact flow of
its generator.

\paragraph{Fixed-step objectives differ.}

Consider the translation $F_{\theta,h}(x)=x+h\theta+h^2$, which is
invertible everywhere. Its generator is $v_\theta=\theta$, but its
finite-step quotient is $q_{\theta,h}=\theta+h$. For a zero target,
the finite-step risk is minimized at $\theta=-h$, whereas the
generator risk is minimized at $\theta=0$. This example shows that
invertibility and $F_{\theta,0}=\mathrm{id}$ alone do not justify
replacing the quotient used at deployment with the zero-step
derivative in CET.


\end{document}